\documentclass[11pt,a4paper]{article}
\usepackage{times,latexsym}
\usepackage{url}
\usepackage[T1]{fontenc}

\usepackage[acceptedWithA]{tacl2018v2} % ,copyedit,nohyperref

\usepackage{amsmath}
\usepackage{tikz-dependency}
\DeclareMathOperator*{\argmax}{arg\,max}

\usepackage{amssymb}% http://ctan.org/pkg/amssymb
\usepackage{pifont}% http://ctan.org/pkg/pifont

\newcommand{\Prob}{\mathbb{P}}%

\usepackage[english]{babel}
\usepackage[utf8]{inputenc}
\usepackage{bm}
\usepackage{graphicx}
\usepackage{tikz}
\usepackage{xcolor}
\usepackage{url}
\usepackage{rotating}
\usepackage{multirow}

\usepackage[T1]{fontenc}

\usepackage{pslatex}
\usepackage[english]{babel}
\usepackage[utf8]{inputenc}
\usepackage{amsmath}
\usepackage{bm}
\usepackage{graphicx}
\usepackage{tikz}
\usepackage{xcolor}
\usepackage{url}
\usepackage{rotating}
\usepackage{amssymb}

\usepackage{amsthm}

\allowdisplaybreaks

\newcounter{theorem}

\newtheorem{corollary}[theorem]{Corollary}

\newtheorem{defin}[theorem]{Definition}

\newtheorem{lemma}[theorem]{Lemma}
\newtheorem{thm}[theorem]{Theorem}

\newcommand{\key}[1]{\textbf{#1}}

\newcommand{\soft}[1]{}
\newcommand{\nopreview}[1]{}

\title{Theoretical Limitations of Self-Attention in Neural Sequence Models}
\author{Michael Hahn \\ Stanford University \\ {\sf mhahn2@stanford.edu}}
\begin{document}
\maketitle
\begin{abstract}
Transformers are emerging as the new workhorse of NLP, showing great success across tasks.
Unlike LSTMs, transformers process input sequences entirely through self-attention.
Previous work has suggested that the computational capabilities of self-attention to process hierarchical structures are limited.
In this work, we mathematically investigate the computational power of self-attention to model formal languages.
Across both soft and hard attention, we show strong theoretical limitations of the computational abilities of self-attention, finding that it cannot model periodic finite-state languages, nor hierarchical structure, unless the number of layers or heads increases with input length.
These limitations seem surprising given the practical success of self-attention and the prominent role assigned to hierarchical structure in linguistics, suggesting that natural language can be approximated well with models that are too weak for the formal languages typically assumed in theoretical linguistics. 
\end{abstract}

\section{Introduction}

Transformers are emerging as the new workhorse of NLP, achieving the state-of-the-art in tasks such as language modeling, machine translation, and creating pretrained contextualized word embeddings.
Eschewing recurrent computations, transformers are entirely based on self-attention, performing their computations largely in parallel.
This enables them to scale to very long sequences \cite{vaswani2017attention,dai2019transformer,child2019generating}.
On the other hand, it has been suggested that this limits their expressiveness, as they cannot process input sequentially \cite{tran2018importance,dehghani2018universal,shen2018disan,chen2018best,hao2019modeling}.
One aspect thought to be challenging for sequence models is hierarchical structure and recursion.
Hierarchical structure is widely thought to be essential to modeling natural language, in particular its syntax~\cite{everaert2015structures}.
Consequently, many researchers have studied the capability of recurrent neural network models to capture context-free languages (e.g., \citet{kalinke1998computation,gers2001lstm,gruning2006stack,weiss2018practical,sennhauser2018evaluating,korsky2019computational}) and linguistic phenomena involving hierarchical structure (e.g., \citet{linzen2016assessing,gulordava2018colorless}).
Some experimental evidence suggests that transformers might not be as strong as LSTMs at modeling hierarchical structure~\cite{tran2018importance}, though analysis studies have shown that transformer-based models encode a good amount of syntactic knowledge (e.g., \citet{clark2019bert,lin2019open,tenney2019bert}).

In this work, we examine these questions from a theoretical perspective, asking whether models entirely based on self-attention are theoretically capable of modeling hierarchical structures involving unbounded recursion.
Formally, we study their ability to perform two computations that are thought to be essential to hierarchical structure:
First, their ability to correctly \key{close brackets}, a basic problem underlying all nonregular context-free languages and formalized by the \key{\textsc{Dyck}} language \cite{chomsky1963algebraic}.
Second, their ability to \key{evaluate iterated negation}, a basic component of the task of evaluating logical formulas, amounting to evaluating the \key{\textsc{Parity}} of bitstrings.
We show that neither of these problems can be solved by transformers and similar models relying entirely on self-attention, unless the number or size of parameters increases with the input length.
Besides representing basic building blocks of hierarchical structure, these languages also represent large classes of regular and context-free languages, meaning that our results carry over to classes of other formal languages.
Our results therefore also yield more generally limitations on the ability of self-attention to model finite-state languages and context-free languages.

While theoretical work has investigated the power of recurrent neural networks in depth (e.g., \citet{siegelman1991neural, bengio1994learning, weiss2018practical, miller2018recurrent, merrill2019sequential}), the theoretical study of self-attention has begun only recently \citep{perez2019turing,hsieh2019robustness}.
Our study provides the first theoretical results on limitations in the power of self-attention.
We will provide results both for hard and soft attention settings, using different proof methods.
Our results are strongest in the hard attention setting, holding without further assumptions on activation functions and parameter norms.
In the soft attention settings, we still obtain results assuming smoothness of activation functions as used in practical implementations.

After discussing related work (Section~\ref{sec:related}), we introduce self-attention (Section~\ref{sec:def-selfatt}) and two fundamental formal languages representing regular and context-free languages (Section~\ref{sec:langs}).
We then prove that self-attention cannot model these languages using either hard (Section~\ref{sec:hard}) or soft (Section~\ref{sec:soft}) attention.
Finally, we discuss our results (Section~\ref{sec:discussion}).

\section{Related Work}\label{sec:related}
\paragraph{Prior Work on Self-Attention}
Transformers were proposed by \citet{vaswani2017attention}, previous related work using self-attention includes \citet{cheng2016long,parikh2016decomposable,paulus2017deep,lin2017structured}.
It has been a recurrent suggestion in the literature that transformers, relying entirely on self-attention, are restricted computationally, as they cannot process their input sequentially.
\citet{dehghani2018universal} suggested that %transformers are restricted computationally, motivating their proposal of a variant with adaptive and unbounded number of layers.
%They argued that
transformers cannot compute functions that require sequential processing of input, without providing further details or proofs.
Similarly, \citet{shen2018disan,chen2018best,hao2019modeling} have introduced extensions of transformers with recurrence, citing similar intuitions about limitations of transformers.
Our results provide the first explicit formalization of these limitations.

A few studies have experimentally tested the abilities of transformers to learn structures.
Most related to our work, \citet{tran2018importance} compared the ability of transformers and LSTMs to learn hierarchical structure, specifically English subject-verb agreement and evaluating logical formulas.
Their experimental results suggested that LSTMs are better at learning hierarchical structure.
\citet{yang2019assessing} experimentally investigated the power of self-attention to extract word order information, finding differences between recurrent and self-attention models; however, these were modulated by the training objective.
\citet{lin2019open} and \citet{tenney2019bert} show that BERT \cite{devlin2018bert} encodes syntactic information.

Theoretical study of transformers was initiated by \citet{perez2019turing}, who theoretically studied the ability of Seq2Seq transformers to emulate the computation of Turing machines.
While we consider incremental modeling of sequences, where the number of computation steps is bounded by the input length $n$, they study the setting in which the transformer computes an unbounded number of autoregressive decoding steps, not bounded in the input length $n$.
Even more recently, and more closely related to our interest here, \citet{hsieh2019robustness} studied the adversarial robustness of transformers.
While they focused on experiments on NLP tasks, they also provided a theoretical analysis, showing that a single self-attention layer with a single head will be robust against input perturbations, assuming that input embeddings are drawn uniformly from the unit sphere.
One of our results, Lemma~\ref{lem:soft-tech}, can be seen as considerably widening the scope of their result, both by avoiding distributional assumptions, and by applying to transformers with arbitrary numbers of heads and layers.

\paragraph{Investigating the Power of Sequence Modeling Architectures}
The computational power of recurrent neural networks has been a focus of study.
A particular focus has been on their ability to learn non-regular context-free languages, thought to provide simple models of recursion and hierarchical structure as found in natural language.

A range of studies has experimentally examined the ability of recurrent networks to model counter languages such as $a^nb^n$ \cite{kalinke1998computation,gers2001lstm,cartling2008implicit,weiss2018practical,suzgun2019evaluating}.
Other work has experimentally studied the performance of recurent architectures on learning to recognize well-bracketed strings, a similar but more challenging problem \cite{sennhauser2018evaluating,skachkova2018closing,bernardy2018can}.
Beyond modeling formal languages, another line of work has studied the ability of LSTMs to model hierarchical structure as occurring in realistic natural language data \cite{linzen2016assessing,gulordava2018colorless}.

Recently, \citet{merrill2019sequential} and \citet{korsky2019computational} theoretically studied several types of recurrent networks. \citet{merrill2019sequential} showed that -- in the finite precision setting -- LSTMs recognize a subset of the counter languages, whereas GRUs and simple RNNs recognize regular languages.
\citet{korsky2019computational} showed, among other results, that arbitrary-precision RNNs can emulate pushdown automata, and can therefore recognize all deterministic context-free languages.

A related, though different, strand of research has investigated the power of neural networks to model Turing machines.
A classical result~\cite{siegelman1991neural} states that -- given unlimited computation time -- recurrent networks can emulate the computation of Turing machines.
Very recently, \citet{perez2019turing} have shown the same result for both (argmax-attention) Transformers and Neural GPUs.
The crucial difference between these studies and studies of language recognition is that, in these studies, the networks are allowed to perform unbounded recurrent computations, arbitrarily longer than the input length.

\section{Self-Attention}\label{sec:def-selfatt}
Here we define self-attention as used in Transformers, following \citet{vaswani2017attention}, with some changes in the notation to simplify arguments in our proofs.
We have an input ${\bf x} = x_1 \dots x_n$, where all $x_i$ come from some finite alphabet $\mathcal{V}$, and $x_n$ is an end-of-sequence symbol.
This input is then encoded into a sequence of input embeddings $v_1,\dots,v_n$ using some embedding map $\mathcal{V} \rightarrow \mathbb{R}^k$. 
We furthermore have a sequence $p_1, p_2, \dots$ of positional embeddings $p_i \in \mathbb{R}^k$. These are independent of the input ${\bf x}$, and can be computed through some predefined scheme, or could be learned for each position occurring in the training data \citep{vaswani2017attention}.
Input and positional embeddings are combined (e.g., via addition or concatenation) to vectors $y_i^{(0)} = f(v_i, p_i)$ ($i=1, \dots, n$), which we will refer to as Layer 0.

A transformer has a fixed number $L$ of \key{layers}; the \key{activations} $y_i^{(k)}$ at position $i$ of the $k$-th layer ($k=1, \dots, L$) are defined as follows.
Each layer has a set of $H$ \key{attention heads}; we first compute attention scores for the $h$-th head:
\begin{equation}
    a_{i,j}^{(k,h)} = f^{att}_{k,h}\left(y_i^{(k-1)}, y_j^{(k-1)}\right)
\end{equation}
where $f^{att}_{k,h}$ combines the activations from the previous level into an attention score.
This can be implemented e.g. using dot product or additive attention.
Specifically, the implementation described by \citet[p. 5]{vaswani2017attention} linearly transforms the position-wise activations $y_i^{(k-1)}$ separately into `query' vectors $Q y_i^{(k-1)}$ and `key' vectors $K y_i^{(k-1)}$ (for some parameter matrices $K, Q$); 
the attention score $a_{i,j}^{(k,h)}$ is then implemented as a scaled dot product of query $Q y_i^{(k-1)}$ and key $K y_j^{(k-1)}$.

The  activation of the head is computed by weighting according to attention weights $\hat{a}_{i,j}^{(k,h)}$:
\begin{equation}
    b_{i,k,h} = \sum_{j=1}^n \hat{a}_{i,j}^{(k,h)} y_j^{(k-1)} 
\end{equation}
We note that the implementation described by \citet{vaswani2017attention} first linearly transforms the activations $y_j^{(k-1)}$ into `value vectors' before multiplying with $ \hat{a}_{i,j}^{(k,h)}$; this is mathematically equivalent to applying this linear transformation to $b_{i,k,h}$ as part of the map $f^{act}$ we describe below.

In the soft attention version, these weights $\hat{a}_{i,j}^{(k,h)}$ are obtained by the softmax operation: $\hat{a}_{i,\cdot}^{(k,h)} = \operatorname{softmax}(a_{i,\cdot}^{(k,h)})$.
In the hard attention variant \cite{perez2019turing}, one takes the actual maximum attention values:
$\hat{a}_{i,j}^{(k,h)} = \delta_{j, \argmax_{j'} a_{i,j'}^{(k,h)}}$.\footnote{When there are multiple positions with maximal attention weight, we will assume that the one occurring first in the sequence is chosen. Our analysis also works under other schemes of resolving ties, such as random selection.}

The  per-position activations are then computed as
\begin{equation}
    y_i^{(k)} := f^{act}(y_i^{(k-1)}, b_{i,k,1}, \dots, b_{i,k,H})
\end{equation}
where $f^{act}$ is implemented as a fully-connected feedforward network with a skip-connection \cite{vaswani2017attention} from $y_i^{(k-1)}$ to $y_i^{(k)}$.

\paragraph{Hard and Soft Attention}
There is a choice between soft attention and hard attention \cite{shen2018reinforced,perez2019turing}.
The one prior theoretical study of transformers~\cite{perez2019turing} assumes hard attention.
In practice, soft attention is easier to train with gradient descent; however, analysis studies suggest that attention often concentrates on one or a few positions in trained transformer models \cite{voita2019analyzing,clark2019bert} and that the most important heads are those that clearly focus on a few positions~\cite{voita2019analyzing}, suggesting that attention often behaves  like hard attention in practice. 
We will  examine both hard (Section~\ref{sec:hard}) and soft (Section~\ref{sec:soft}) attention.

\paragraph{Formalizing Language Recognition}
We consider the problem of language recognition, the task of classifying input strings as belonging to or not belonging to a formal language.
Following~\citet{weiss2018practical}, we formalize this as the sequence-to-sequence task of mapping words to labels $1$ (`in the language') and $0$ (`not in the language').
Following the construction of transformers in sequence-to-sequence tasks \cite{vaswani2017attention}, we compute a softmax probability vector for this label from the last activation $y_{n}^{(L)}$, obtained after reading the end-of-sequence symbol.

\section{Regular and Context-Free Languages}\label{sec:langs} 

We will analyze the ability of transformers to recognize regular and context-free languages, using two prominent representatives.

\paragraph{\textsc{Parity}} is the set of bit strings such that the number of $1$s is even.
This is a very simple regular language, recognized by a finite-state automaton with two states.
The regular languages form the lowest layer of the Chomsky hierarchy, and even simple RNNs can compute all regular languages.
Within the regular languages, a particularly basic class is formed by the \emph{counter-free} or \emph{star-free} languages \cite{mcnaughton1971counter}, which can be expressed by regular expressions using only union, complementation, and concatenation.
In some sense, \textsc{Parity} is the simplest non-counter-free, or \emph{periodic}, regular language.
This means, if transformers cannot compute \textsc{Parity}, they cannot recognize (almost)\footnote{Inability to compute \textsc{Parity} entails that they cannot recognize any regular language whose syntactic morphism is not quasi-aperiodic~\cite[p. 488]{barrington1992regular}.} any regular language that is not counter-free.
In the context of natural language, \textsc{Parity} naturally arises in the context of evaluating logical formulas:
Evaluating iterated negations is tantamount to counting whether the number of nested negations is even or odd.
If transformers cannot compute parity, they also cannot evaluate logical formulas accurately.

\paragraph{\textsc{2Dyck}} is the language of correctly bracketed words consisting of two types of brackets (`(', `)' and `[', `]').
This language is a very simple model of hierarchical structure.
The Chomsky-Sch{\"u}tzenberger theorem \cite{chomsky1963algebraic} states that any context-free language arises from a variant of \textsc{2Dyck} with multiple types of parentheses through intersection with a regular language and homomorphisms.
Consequently, the ability of LSTMs to model languages such as \textsc{2Dyck} has been an object of experimental study \cite{sennhauser2018evaluating,skachkova2018closing,bernardy2018can}.
Our theoretical results will show that transformers are strongly limited in their ability to model \textsc{2Dyck}, including variants with fewer or more types of parentheses.

\section{Results for Hard Attention}\label{sec:hard}

We will start our analysis with the study of hard attention~\cite{perez2019turing}.
We show that hard attention transformers cannot represent \textsc{Parity} or \textsc{2Dyck}. 
To keep the results maximally general, our analysis will use combinatorial arguments and make no assumption about, e.g., activation functions and the norms of parameter matrices.
In fact, we do not even assume that the internal position-wise representations $y_{j}^{(k)}$ in each layer are vector-valued, as opposed to, say, discrete structures.

\begin{figure*}[ht]
    \centering
    \begin{tabular}{cccc}
    (a) & (b) & (c) & (d) \\
    \includegraphics[width=0.23\textwidth]{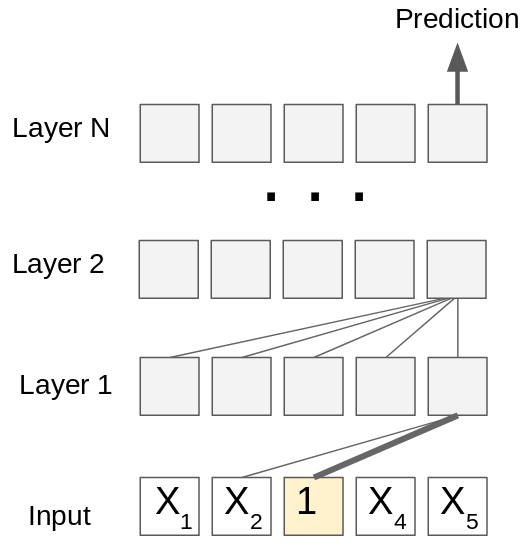} &
        \includegraphics[width=0.23\textwidth]{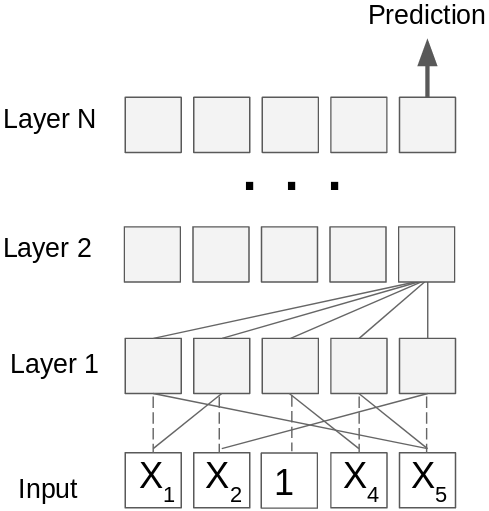}&
    \includegraphics[width=0.22\textwidth]{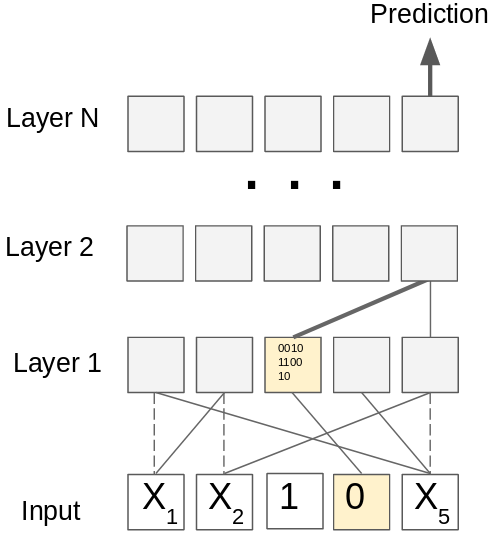} &
        \includegraphics[width=0.23\textwidth]{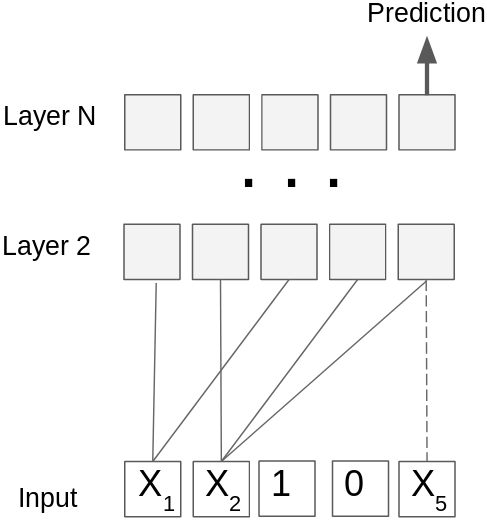}
        \end{tabular}
	\caption{Iteratively reducing the layers of a transformer by fixing a few input symbols. (a) By applying a suitable input restriction, we fix a small number of input symbols, `attracting' attention from the first layer to a few inputs. (b) After this step, Lemma~\ref{lemma:depth-red} ensures that each activation in the first layer only depends on a small number of input symbols that it can attend to (solid connections), plus the input that feeds into it via a skip connection (dashed connections). (c) We again fix a few input symbols in such a way as to `attract' attention of layer-2 heads to some layer-1 activations. As a result, each layer-2 activation only depends on a small number of layer-1 activations, again by Lemma~\ref{lemma:depth-red}. (d) After this step, each layer-1 activation only depends on a few inputs, and we can remove layer 1.
	}
	\label{fig:depth-reduction}
\end{figure*}

We aim to prove that no hard-attention transformer is capable of representing \textsc{Parity} or \textsc{2Dyck}, by constructing -- for any given candidate transformer model -- a set of input words that this model will have to misclassify.
The basic idea (see Figure~\ref{fig:depth-reduction}) behind the proof is that, by fixing a small fraction of the input symbols in a particular way, we can `capture the attention' of the transformer in such a way that it ends up ignoring almost all remaining input symbols.
This shows that the transformer could not have solved a problem such as \textsc{Parity}, where every single input bit matters.

In order to formalize the idea of `fixing'  a few input bits, we introduce the notion of input restrictions:
An \key{input restriction} (short: \key{restriction}) $\rho$ is a family of maps $\rho_n : \{1, \dots, n\} \rightarrow \{*, 0, 1\}$ for $n \in \mathbb{N}$.
An input restriction $\rho$ is applied to a transformer by fixing, when the input length is $n$, the input symbol $x_i$ to the value $\rho_n(i)$ whenever $\rho_n(i) \in \{0, 1\}$.
The output value of the resulting transformer only depends on those inputs $x_i$ such that $\rho_n(i) = *$.

The idea of using such input restrictions has been successful in the theory of Boolean circuits~\cite{furst1984parity,hastad1994optimal}.
In particular, \citet{furst1984parity}  famously used it to prove that polynomial-size bounded-depth Boolean circuits with $\wedge, \vee$, and $\neg$ gates cannot compute \textsc{Parity}.
We describe a new method to prove existence of suitable restrictions appropriate to transformers, as the proof approaches from the Boolean circuit literature do not seem to generalize to networks with real-valued activations.

The following result formalizes the claim that any transformer can be forced to ignore input bits by fixing some inputs in a particular way:
\begin{thm}\label{thm:hardmax-main}
Let any hard attention transformer be given, and let $C \in (0,1)$.
Then there is a restriction $\rho$ and an integer $c > 0$ such that 
$$|\{i \leq n: \rho_n(i) = *\}| \geq Cn$$
(for all sufficiently large $n$) and such that the function computed by the transformer on the restricted input depends only on $\leq c$ inputs, independent of input length $n$.
\end{thm}
We first show how this entails that transformers do not recognize the two formal languages:
\begin{corollary}
Transformers with hard attention cannot model \textsc{Parity} or \textsc{2Dyck}.
\end{corollary}
\begin{proof}
For \textsc{Parity}, after applying a restriction, the transformer's output depends on $c$ inputs.
An input of sufficiently large size $n$ thus has unrestricted inputs that do not influence the output.
But flipping a single input bit changes the value, so the transformer's output cannot match membership in \textsc{Parity} beyond chance for such $n$.

For \textsc{2Dyck}, we show that hard attention transformers cannot even solve the simpler variant \textsc{1Dyck} with a single bracket type (`(', `)').
We first restrict the first $0.2n$ input positions to `(' and the last $0.2n$ positions to `)'.
After then applying the restriction provided by the theorem with $C=0.9$, the resulting restricted input will still be compatible with both well-bracketed and non-well-bracketed inputs, but the prediction will depend only on a bounded number of positions.
As the prediction depends on only a bounded number of positions, this shows the transformer could not recognize \textsc{1Dyck}, and thus also not \textsc{2Dyck}.
\end{proof}

\paragraph{Discussion}
It may be instructive to compare to similar languages that \emph{can} be modeled by hard-attention transformers.
First, $1^*$ (over the alphabet $\{0,1\}$) is the regular language of words that have only ones and no zeroes; its minimal automaton has two states, like \textsc{Parity}.
A transformer can recognize this by having an attention head that attends to a position with zero input if it exists, and rejects if the head found such a position.
Second, $a^nb^n$ is a very basic context-free language.
It can be recognized using suitable positional embeddings by (1) having one head attend to the largest position $n$, (2) using this information to attend to any $b$ at position $<n/2$ or any $a$ at position $\geq n/2$. If such a symbol is found, the model rejects, else it accepts.
A crucial difference between these languages and \textsc{Parity} / \textsc{2Dyck} is that fixing a few inputs in any part of an input string can easily force nonmembership, e.g. a single 0 for $1^*$, and an $a$ in the second half for $a^nb^n$.
Therefore, such simple languages are immune to the depth reduction method, and indeed \emph{can} be modeled perfectly with self-attention.

In general, the depth reduction method applies to languages that are sufficiently \emph{sensitive}: If, for some $C \in (0,1)$, fixing $Cn$ input symbols cannot force a word to be inside or outside of the language, then hard-attention transformers cannot recognize this language.
Sensitivity of functions 
has been studied in computational complexity~\cite{boppana1997average,gopalan2016smooth} and more recently linked to generalization in feedforward networks~\cite{de2018deep}.
We intend to investigate these connections in future work.

\paragraph{Proof Idea of the Theorem}
Our approach for proving Theorem~\ref{thm:hardmax-main} will be to construct input restrictions in a layerwise manner, starting from layer 1. 
In order for this construction to go through, the main challenge is to construct a suitable restriction at a given layer:
As shown in Figure~\ref{fig:restr}, this restriction should only affect a few input bits (about $(1-C^{1/L})n$ many input bits), while forcing each attention head in the first layer to ignore all but $c$ input bits.
Perhaps surprisingly, this is possible; the idea is to fix input bits that achieve high attention scores for several heads, so that input bits that cannot achieve such high attention scores will be ignored.

\begin{figure}[ht]
    \centering
    \begin{tabular}{cccc}
    (a) & (b) \\
    \includegraphics[width=0.22\textwidth]{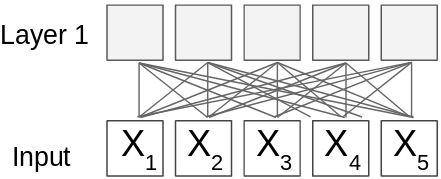} &
        \includegraphics[width=0.22\textwidth]{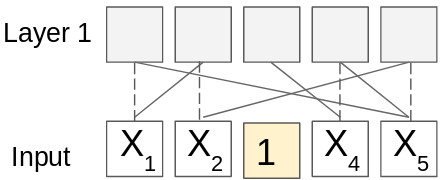}&
        \end{tabular}
	\caption{Finding a good input restriction: (a) Every attention head in the first layer could potentially attend to any input bit. (b) Perhaps surprisingly, one can fix a small number of input bits in such a way that each layer-1 attention head can only possibly attend to $c$ (here, $c=1$) inputs, and ignores all other inputs. Each activation vector $y_j^{(1)}$ in the first layer then only depends on the $H\cdot c$ inputs that its $H$ (here, $H=1$) attention heads can attend to, plus the input $x_j$ that feeds into it via a skip-connection.}
	\label{fig:restr}
\end{figure}

Once we have shown that such a restriction always exists, we can use this technique to iteratively remove layers, as illustrated in Figure~\ref{fig:depth-reduction}:
After we have applied the first such restriction, each of the heads in the first layer will only depend on a bounded number $c$ of input positions.
In the second step, we apply the same argument to the heads in the second layer, so that each head in the second layer only depends on a bounded number $c'$ of heads in the first layer.
After this step, we can collapse the first layer into a collection of feedforward networks that transform a bounded number $\leq cc'$ of input positions into an activation $y_i^{(0)}$ of the lowest layer.
After this step, the first layer has been entirely removed.
Iterating this argument, we remove all layers until the prediction output only depends on a bounded number of input positions, bounded independently of input length.

We now make these ideas formal.
After the removal of the first layer of a transformer, the resulting structure is not a transformer any more, as each head in the lowest layer now depends on a \emph{combination} of input positions.
We introduce a technical definition to make this concept precise:

\begin{defin}[$c$-Transformer]
Let $c$ be a positive integer. A $c$-transformer with $L$ layers is one in which the layer-0 activations $y_j^{(0)}$ depend on the embeddings not just at one position $j$, but are a function of the embeddings at $\leq c$ input positions:
\begin{equation}
    y_j^{(0)} = f^{inp}_{n,j}((v_{i_1^{j,n}}, p_{i_1^{j,n}}), \dots, (v_{i_c^{j,n}}, p_{i_c^{j,n}} ))
\end{equation}
for some indices ${i_s^{j,n}} \in \{1, \dots, n\}$ ($s = 1, \dots, c$).
\end{defin}

With this technical notion, we show that we can reduce layers, iteratively removing the lowest layer until no self-attention layer is left:

\begin{lemma}[Depth Reduction Lemma]\label{lemma:depth-red}
Given a $c$-transformer with $L$ layers, and some restriction $\rho$ such that
\begin{equation}
|\{i \leq n: \rho_n(i) = *\}| \geq Cn
\end{equation}
($C \in (0,1]$)
for all sufficiently large $n$.
Choose any $C' < C$.
Then there is a restriction $\rho'$ 
such that
\begin{equation}
|\{i \leq n: \rho'_n(i) = *\}| \geq C'n
\end{equation}
for all sufficiently large $n$, 
and such that the resulting function is computed by a $( c\cdot(2^ckH+1))$-transformer with $L-1$ layers, for some integer $k$ (depending on $C'$), where $H \geq 1$ is the number of attention heads at each layer and position.
\end{lemma}
The lemma implies Theorem~\ref{thm:hardmax-main}:
\begin{proof}[Proof of Theorem~\ref{thm:hardmax-main}]
The output of the transformer is determined by the last activation $y_{n}^{(L)}$.
Apply the Depth Reduction Lemma iteratively, choosing the constants $C'$ in the lemma appropriately, until only the zero-th layer remains.
Then, after applying the resulting restriction, the final activation $y_{n}^{(L)}$ is now computed by $y_{n}^{(0)}$, which is determined by a bounded number of input bits.
\end{proof}

\subsection{Proving the Depth Reduction Lemma}
In this section, we will prove the Depth Reduction Lemma.
We construct the restrictions $\rho'_n$ separately for each $n$, on the basis of the given restriction $\rho_n$.
In this process, we will only \emph{restrict additional bits}, that is, the only case in which $\rho_n'(i)$ can be different from $\rho_n(i)$ is that $\rho_n'(i)$ may be $0$ or $1$ where $\rho_n(i)$ was $*$.
The construction proceeds in three stages $\rho_n^{(1)}$, $\rho_n^{(2)}$, and $\rho_n^{(3)} = \rho_n'$, which all may restrict additional bits.
At the end, we verify that the conclusion of the Depth Reduction Lemma is satisfied for the resulting restriction $\rho_n'$.

Throughout the proof, we will need a few parameters independent of $n$: First, we need an integer $k$ that has to be sufficiently large for the proof to succeed, and will be fixed later in the proof.
Second, we need parameters $\eta \in (0, \frac{1}{2})$, $q \in (0,1)$ and $\delta > 0$; we will also fix the specific values later in the proof.

\paragraph{Stage 1}
We start from $\rho_n$ and first modify it into a restriction $\rho^{(1)}_n$ such that each input bit serves as an input to at most $\leq \frac{1}{\eta} c/C$ many different layer-0 heads, when applying $\rho^{(1)}_n$.
Assume the number of input bits feeding into more than $\frac{1}{\eta} c/C$ different layer-0 activations is $\geq \eta Cn$.
Then the number of pairs of input bits and depending layer-0 activations is $>\eta Cn \cdot \frac{1}{\eta} c/C = nc$.
But there are at most $nc$ such pairs, because there are $n$ layer-0 activations, each of which depends on $\leq c$ inputs.
So the number of input bits with $> \frac{1}{\eta} c/C$ depending layer-0 heads is $\leq \eta Cn$.
We can obtain $\rho^{(1)}_n$ from $\rho_n$ by restricting these input bits to some fixed value in $\{0, 1\}$ (it doesn't matter which one), and the set $\{i \leq n: \rho^{(1)}_n(i) = *\}$ still has at least $(1-\eta) C n$ elements, for all sufficiently large $n$.

\paragraph{Stage 2}
We now describe the second stage.
We write $(h,i)$ for a layer-1 attention head $h$ ($h=1,\dots, H$) at position $i$ ($i=1, \dots, n$).
Fix such a head $(h,i)$.
As $y^{(0)}_i$ depends on $\leq c$ input bits, it can take on at most $\leq 2^c$ possible values.
For each possible value $z$,  and each position $j \in \{1, \dots, n\}$, we compute the maximum possible attention value that can be achieved for this pair:
\begin{equation}\label{eq:max-att}
\max_{x_1\dots x_n\ :\ y^{(0)}_i=z} f^{att}_{1,h}(z, y^{(0)}_j)
\end{equation}
considering only inputs $x_1\dots x_n$ that are compatible with the restriction $\rho^{(1)}_n$ constructed at Stage 1.
For each value $z$, we order the positions $\{1, \dots, n\}$ downwards by this value, obtaining a sequence $j_1^{(z)}, \dots, j_n^{(z)}$ for each layer-1 attention head $h$ at a position $i$ and each possible value $z$ of $y^{(0)}_i$ (In the case of ties, we order by  position, by Footnote 1).
For each layer-1 attention head and each $z$, we select a sequence $1 \leq i_1^{(z)} < i_2^{(z)} < \dots < i_{k}^{(z)} \leq n$ such that (1) for each $i_s^{(z)}$, there is at least one input $x_q$ that only feeds into the activation at position $j_{i_s^{(z)}}^{(z)}$ and no $j_{i_{s'}^{(z)}}^{(z)}$ ($s\neq s'$), and (2) $i_{k}^{(z)}$ is minimal, i.e. there is no subsequence with smaller $i_{k}^{(z)}$ that also satisfies (1).
This construction is visualized in an example in Figure~\ref{fig:selecting}.
Such a subsequence exists unless $n \leq ck$, in which case the Depth Reduction Lemma is already satisfied for this input length $n$.

If $z$ is a possible value of the activation $y^{(0)}_i$, then we
say that a pair $((i,h),z)$, of a head $h$ at position $i$ and a possible value $z$ of $y^{(0)}_i$, is \key{satisfied} if one of the layer-0 activations $y^{(0)}_{i_s^{(z)}}$ ($s \in \{1, \dots k\}$) is fixed by $\rho^{(1)}_n$ to the value achieving the maximum attention value (\ref{eq:max-att}).
Also, we say that $(h,i)$ is satisfied if each $((h,i),z)$ is.
The idea behind this definition is: If $((h,i),z)$ is satisfied, then there are at most $k$ different layer-0 heads that this head could attend to when applying $\rho_n'$, assuming that $y^{(0)}_i$ takes the value $z$.
As a consequence, a satisfied head can only depend on $ c\cdot (2^ck+1)$ many input bits.
Our aim will be to construct $\rho_n'$ so that each layer-1 head is satisfied.

\begin{figure}[ht]
    \centering
    \begin{tabular}{cccc}
    (a) & (b) \\
    \includegraphics[height=0.1\textheight]{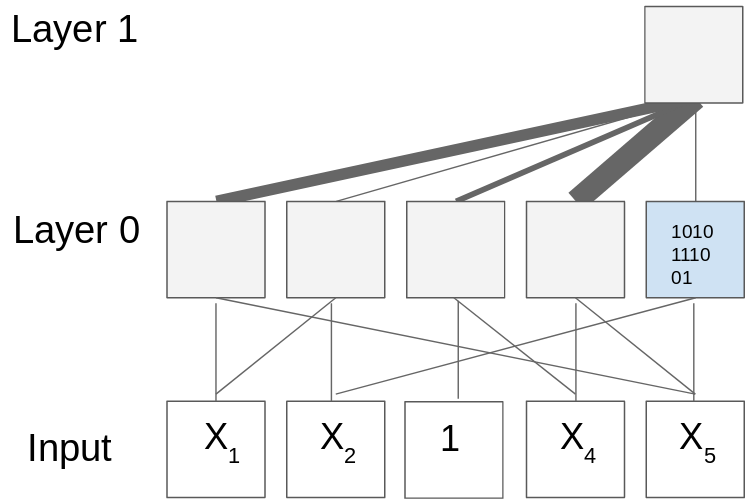} &
        \includegraphics[height=0.1\textheight]{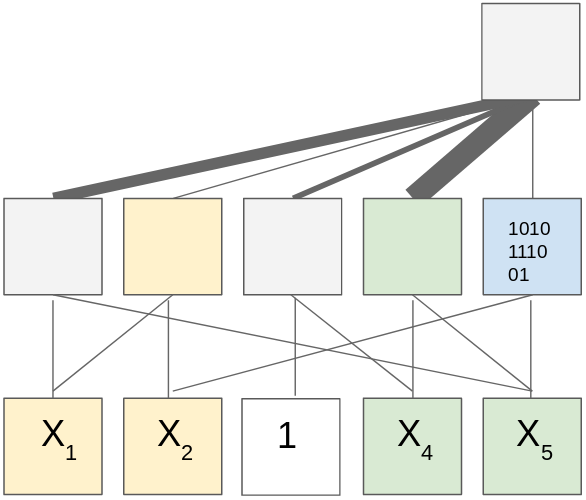}&
        \end{tabular}
	\caption{Selecting the sequence $i_{1}^{(z)} \dots i_{k}^{(z)}$: We have a $c$-transformer with $c=2$, i.e., each Layer 0 activation only depends on at most two input bits.
	(a)  We fix a head in Layer 1 at position $i$ (here, $i=5$), and some value $z$ for $y^{(0)}_i$ (blue activation). For each other Layer-0 activation $y^{(0)}_j$, we compute the maximal possible attention value between that activation and the Layer 1 head, assuming the fixed value $z$ for $y^{(0)}_i$ -- these maximum attention values are visualized by the thickness of the different lines.
	(b) We select $k = 2$ activations from Layer 0, marked in yellow and green. For each of these, there is at least one (in fact, two in the example) input bits (also marked in yellow and green) that feed into this one and no other selected activation.}
	\label{fig:selecting}
\end{figure}

A layer-1 head $k$-\textbf{depends} on some input $x_i$ if $\rho_n(i) = *$ and $x_i$ appears as an input to some $j_r^{(z)}$ for $r \leq i_k^{(z)}$, for some $z$.
Because $i_k^{(z)}$ is minimal, a layer-1 head $k$-depends on an input if and only if that input appears as an input to some $j_{i_s^{(z)}}$ ($s \leq k$).
In particular, a layer-1 head $k$-depends only on $\leq 2^c ck$ input variables.
Two layer-1 head are $k$-\textbf{neighbors} if some $j_{i_s^{(z)}}$ for one and $j_{i_{s'}^{(z')}}$ for the other both $k$-depend on some input bit $x_l$.

We will construct $\rho'_n$ using probabilistic arguments over randomly chosen restrictions.
For this approach to succeed, we require a  sufficient amount of independence between the activations of different heads in layer 1.
We thus need to ensure that the number of $k$-neighbors of each head is bounded.
Recall $\eta \in (0,\frac{1}{2})$, and let $H$ be the number of attention heads in each position of layer 1.

We modify $\rho^{(1)}_n$ into $\rho^{(2)}_n$ so that each layer-0 head has at most $\leq 2^c kH$ many $k$-depending unsatisfied layer-1 heads.
Assume that indeed some layer-0 head  has more than $2^c kH$ many $k$-depending unsatisfied layer-1 heads. % $(h,i)$ such that some $((h,i),z)$ does not yet satisfy the Depth Reduction Lemma.
By fixing $\leq c$ input bits and appealing to the Pigeonhole principle, we can fix this head to a value that achieves the maximum attention value for at least $> kH$ many of these $k$-depending layer-1 heads.
Let $\rho^{(2)}_n$ be the restriction resulting from adding this to $\rho^{(1)}_n$.
Once we have done this, 
$\{i \leq n: \rho_n^{(2)}(i) = *\}$ still has at least $(1-\eta) C n - c$ elements, and more than $kH$ many additional pairs $((h,i),z)$ are now also satisfied.
We then repeat the selection of the sequence $j_1^{(z)}, \dots, j_n^{(z)}$ (substituting $\rho^{(2)}_n$ for $\rho^{(1)}_n$ in the definition), and repeat the construction described here, to restrict additional input bits in $\rho^{(2)}_n$.
We iterate this procedure until no layer-0 head has $>  2^c kH$ many $k$-depending unsatisfied layer-1 heads $(h,i)$. % such that some $((h,i),z)$.
This procedure can be iterated at most until each layer-1 head is satisfied, that is, at most $\frac{2^c H n}{kH} = \frac{2^c n}{k}$ times.
Let $U$ be the number of times this procedure is iterated ($U \leq \frac{2^c n}{k}$).
At the end, $\{i \leq n: \rho_n^{(2)}(i) = *\}$ has at least $(1-\eta) C n - cU \geq \left((1-\eta) C  - \frac{2^c c}{k}\right) n$ elements.
By choosing $k$ so large that $\frac{2^c c}{k} \leq \eta$, we find that $\{i \leq n: \rho^{(2)}_n(i) = *\}$  still has at least $(1-2\eta) C n$ many elements.
Once this is completed, each layer-0 head has at most $\leq 2^c kH$ many $k$-depending unsatisfied layer-1 heads.
Thus each input bit now has at most $\leq \frac{2^c}{\eta}kcH/C$ many $k$-depending unsatisfied layer-1 heads.
Consequently, every unsatisfied layer-1 head has at most $f \leq \frac{2^{2c}}{\eta}c^2k^2H/C$ many unsatisfied $k$-neighbors.

\paragraph{Stage 3}
In order to construct the third and final restriction $\rho^{(3)}_n$, we apply the ``probabilistic method'': We define a probability distribution over restrictions $\rho^{(3)}_n$, and show that the probability assigned to restrictions of the type we require is strictly greater than zero, showing that such a restriction exists.
For each input length $n$, define the distribution over restrictions $\rho_n^{(3)}$ that independently assigns to each input position $i \in \{1, \dots, n\}$ the symbol $1$ or $0$ with probability $q/2$ each ($q \in (0,1)$ chosen later), and $*$ with probability $1-q$.
On those input bits where $\rho_n^{(2)}(i) \neq *$, we restrict this random restriction to agree with $\rho_n^{(2)}(i)$.
For an layer-1 attention head $(h,i)$ and for each value $z$ (there are at most $2^c$ such), define $X_{i,h}^{(z)}$ to be the event that, for this head, none of $y_{j_{i_1^{(z)}}}^{(0)}, \dots, y_{j_{i_k^{(z)}}}^{(0)}$ are fixed to the value that produces the highest attention weight.
Define $X_0$ to be the event that more than $(1+\delta)q$ of the input bits that $\rho_n^{(2)}$ maps to $*$ are set to $0/1$ by $\rho_n^{(3)}$ (where $\delta \in (0,1)$, to be fixed later).
Our goal will be to show that a nonzero amount of probability mass is assigned to restrictions $\rho_n'$ avoiding all events $\{X_0\} \cup \{X_{i,h}^{(z)} : i, z\}$.

First, a Chernoff bound gives~\cite[Theorem 4.4]{mitzenmacherprobability}
\begin{equation}
\Prob(X_0) \leq    \exp\left(-\delta^2q(1-2\eta)Cn / 3\right)
\end{equation}
since $\rho_n^{(2)}$ had $\geq (1-2\eta)Cn$ unrestricted input bits after Stage 2.

Second, we show that the probability of $X_{i,h}^{(z)}$ ($i=1,2,\dots, n$, $h=1, \dots, H$) decays exponentially in $k$.
First, if $((h,i),z)$ is already satisfied after Stage 2, then $\Prob(X_{i,h}^{(z)}) = 0$.
Else, fixing $z$ for ease of notation, let $Y_{i,h}^t$ ($t=1,\dots,k$) be the event that the layer-0 activation $y_{j_{i_t^{(z)}}}^{(0)}$ is not fixed to the value that produces the highest attention weight, for the given attention head $(h,i)$.
Note that $X_{i,h}^{(z)} = \bigcap_t Y_{i,h}^t$.
We have $\Prob(Y_{i,h}^s) \leq 1-(q/2)^c \in (0,1)$. 
Any $Y_{i,h}^s$ can be statistically dependent on at most $c \cdot \frac{1}{\eta}c/C = \frac{1}{\eta}c^2/C$ other events $Y_{i,h}^{s'}$, because each input bit has at most $\frac{1}{\eta} c/C$ depending layer-0 heads.
Therefore, there is a set of $\geq \frac{k}{\frac{1}{\eta}c^2/C}$ independent events among these.
Call these $Y_{i,h}^{t_1}, \dots, Y_{i,h}^{\frac{k}{\frac{1}{\eta}c^2/C}}$.
Then $X_{i,h}^{(z)} \subseteq \bigcap_{s=1}^{\frac{k}{\frac{1}{\eta}c^2/C}} Y_{i,h}^{t_s}$, and thus
\begin{equation}
\Prob(X_{i,h}^{(z)}) \leq \prod_{s=1}^{\frac{k}{\frac{1}{\eta}c^2/C}} \Prob(Y_{i,h}^{t_s}) \leq \left(1-(q/2)^c\right)^{\frac{k}{\frac{1}{\eta}c^2/C}}
\end{equation}
for each $i=1,2,\dots, n$ and $h=1, \dots, H$.

In order to conclude that there is a restriction $\rho^{(3)}_n$ avoiding all events $\{X_0\} \cup \{X_{i,h}^{(z)} : i, h, z\}$, we apply the Lov{\'a}sz Local Lemma \cite[Theorem 6.17]{mitzenmacherprobability}.
Each event $X_{i,h}^{(z)}$ is statistically independent of the set $\{X_{(j,h')}^{(z')} : \text{heads } {(j,h')} \text{ and } {(i,h)} \text{ are not $k$-neighbors}\}$.
The complement of this set has cardinality $\leq f= \frac{2^{2c}}{\eta}c^2k^2H/C$, as concluded at the end of Stage 2.
Set $A:=\frac{1}{k^2}$, $B:=\frac{1}{2}$.
By the Lov{\'a}sz Local Lemma, it is sufficient show the following: %, assuming $f \leq $:
\begin{align}\label{eq:lovasz-1}
&\Prob(X_{i,h}^{(z)}) \leq A(1-B)(1-A)^{f} \\ \label{eq:lovasz-2}
&\Prob(X_0)  \leq B (1-A)^{2^cHn}
\end{align}
The Lov{\'a}sz Local Lemma then guarantees that there is some input restriction $\rho^{(3)}_n$ that avoids all events $\{X_0\} \cup \{X_{i,h}^{(z)} : i, h, z\}$.
For~(\ref{eq:lovasz-1}), we need
\begin{align}\label{eq:x1-ineq}
    D &\leq A^{1/k}(1-B)^{1/k}(1-A)^{f/k} 
\end{align}
where $D =  \left(1-(q/2)^c\right)^{\frac{1}{\frac{1}{\eta}c^2/C}} \in (0,1)$.
For the first term on the right, 
\begin{align*}
\lim_{k\rightarrow \infty} A^{1/k} = \lim_{k\rightarrow \infty} \exp\left(-\log(k^2) / k\right) = 1
\end{align*}
Also, $\lim_{k\rightarrow \infty} (1-A)^{f/k}$ equals
\begin{align*}
\lim_{k\rightarrow \infty} \left(1-\frac{1}{k^2}\right)^{\frac{2^{2c}}{\eta}c^2kH/C} = \lim_{k\rightarrow \infty} \left(1-\frac{E^2}{k^2}\right)^{k} = 1
\end{align*}
for $E := \frac{2^{2c}}{\eta}c^2H/C$. So, if we choose $k$ large enough (independently of $n$), the RHS of (\ref{eq:x1-ineq}) can be made arbitrarily close to $1$, in particular, greater than $D$.
In order to also satisfy~(\ref{eq:lovasz-2}), we need
\begin{align*}
\exp\left(-\delta^2q(1-2\eta)C/3\right)  \leq B^{1/n} (1-A)^{2^c H}
\end{align*}
which holds for $n$, $k$ large enough (again, choosing $k$ independent of $n$). 
In conclusion, there exists, for each sufficiently-large $n$, a restriction $\rho^{(3)}_n$ that avoids all events $\{X_0\} \cup \{X_{i,h}^{(z)} : i, z\}$, for some $k$ independent of $n$.
For such a $\rho^{(3)}$, we have
\begin{equation*}
|\{i \leq n: \rho^{(3)}_n(i) = *\}|\geq (1-2\eta)\cdot (1-(1+\delta)q) C n
\end{equation*}
for all sufficiently large $n$.
Then choose $\eta \in (0,\frac{1}{2})$ small, $q \in (0,1)$, and $\delta >0$ (such that $(1+\delta)q \in (0,1)$) in such a way as to achieve $(1-2\eta)\cdot (1-(1+\delta)q) = C'/C$.

After applying $\rho^{(3)}_n$, every layer-1 head $b_{j,1,h}$ depends only on (1) the $c$ input bits feeding into $y_j^{(0)}$, and (2) the $\leq c2^ck$ input bits that the head $k$-depends on.
Thus, each layer-1 activation $y_j^{(1)}$ only depends on $\leq c\cdot (2^ckH+1)$ input bits: There are $\leq H\cdot c \cdot 2^c \cdot k$ input bits that the $H$ different attention heads $k$-depend on, plus a skip-connection from $y_j^{(0)}$, which itself depends on $\leq c$ input bits.
We can thus remove layer 0, convert layer-1 activations $y_j^{(1)}$ into layer-0 activations $y_j^{(0)}$, and obtain a $(c\cdot(2^ckH+1))$-transformer performing the same computation as before when $\rho^{(3)}$ is applied.
This concludes the proof of the Depth Reduction Lemma.

\section{Results for Soft Attention}\label{sec:soft}

In the previous section, we showed that transformers using hard attention are not able to recognize a range of core formal languages.
In this section, we study soft attention.
It turns out that proving limitations as strong as what we found in the hard attention setting would settle a major open problem in computational complexity, and  may therefore be extremely hard to attain with currently available mathematical methods.\footnote{Showing that soft attention transformers cannot achieve perfect accuracy on evaluating Boolean formulas would separate the complexity classes $LTC^0$ and $NC^1$, a widely conjectured but long-open problem in computational complexity.}
This barrier prevents us from proving bounds on the \emph{accuracy} that soft attention transformers can achieve; nevertheless, we will be able to prove limitations on the achievable \emph{cross-entropy} in modeling distributions over the formal languages.
We will use the smoothness of the operations used in transformers to show that any transformer, as inputs get longer, will not be able to robustly model such distributions.
The idea behind the proof is that the impact of any single input symbol on the output of the transformer is small if the input is long:
\begin{lemma}\label{lem:soft-tech}
Let a soft attention transformer be given, and let $n$ be the input length.
If we exchange one input symbol $x_i$ ($i < n$), %a single input bit $x_i$ with the bit $x_i'$
then the change in the resulting activation $y_n^{(L)}$ at the decoder layer is bounded as $\mathcal{O}(\frac{1}{n})$ with constants depending on the parameter matrices.
\end{lemma}
This contrasts with recurrent networks:
Changing a single input can have nonnegligible impact on the final state even for very long input.
E.g., an RNN recognizing \textsc{Parity} through a hidden state that encodes parity of the current prefix will flip its hidden state if a single input bit is flipped.

Lemma~\ref{lem:soft-tech} entails that, as inputs become longer, soft attention transformers cannot achieve good cross-entropies on prediction problems that are very sensitive to individual input symbols:
A Lipschitz-continuous prediction function, such as a ReLU MLP with a softmax output, will not be able to make very different predictions for inputs that are encoded into similar activations $y_n^{(L)}$.

To make all our assumptions explicit, we will assume the following setting, though the results do not depend on the specific details.
For \textsc{Parity}, we consider the distribution over bitstrings generated by a two-state automaton that -- if the number of $1$s emitted so far is even -- terminates with probability $p$, and otherwise emits a $1$ or $0$ with equal probability each.
Given a prefix of a string drawn from this distribution, we ask the transformer to predict the next symbol from $\Sigma = \{0, 1, \textsc{EndOfSequence}\}$.
Note that the next symbol can be \textsc{EndOfSequence} if and only if the prefix has an even number of $1$s.
For \textsc{2Dyck}, we follow the experimental study of \citet{skachkova2018closing} and take the distribution generated by a PCFG that expands $S \rightarrow (S)S$ or $S \rightarrow [S]S$ with probability $p/2$ each, and $S \rightarrow \epsilon$ with probability $1-p$. %To give specific numerical values, we choose $p=1/4$ here.
We ask the model to predict the next character among $\Sigma = \{(,),[, ], \textsc{EndOfSequence}\}$.

\begin{thm}
Let a soft attention transformer be given for \textsc{Parity} or \textsc{2Dyck}.
As $n\rightarrow\infty$, cross-entropy on predicting the next symbol converges to unigram chance level (\textsc{Parity}), or is at least separated from the optimal cross-entropy by some constant $\epsilon >0$  (\textsc{2Dyck}).
\end{thm}

\begin{proof}
First, let us consider \textsc{Parity}.
Exchanging a single bit flips membership in \textsc{Parity}.
Thus, for any $x \in \textsc{Parity}$, there is a string $x' \not\in \textsc{Parity}$, differing only in one bit.
As $x$ and $x'$ differ only in one bit, the transformer's output activations differ by $\mathcal{O}(\frac{1}{n})$.
Therefore, a Lipschitz-continuous prediction function cannot robustly assign different next-symbol probabilities after even and odd numbers of 1s, and cross-entropy will converge to unigram chance level.

For \textsc{2Dyck}, consider a string $x$ of length $n$, known to be the prefix of a word generated by the PCFG.
One can show that there is a constant $P_0 \in (0,1)$ (dependent on $p$ but not $n$), such that $x$ both ends with a closing bracket, and is unbalanced, with probability $\geq P_0$.\footnote{One can show this formally using a Markov chain argument.
Let the \key{height} $H(x)$ of a word $x$ be the number of opening brackets minus the number of closing brackets in $x$.
When iteratively sampling a symbol sequence using a pushdown automaton for \textsc{2Dyck}, the height $H_{n}$ of the prefix $x$ up to length $n$ forms a Markov chain taking values in $\mathbb{N}$. 
The prefix $x$ is unbalanced if and only if $H_n>0$, this is always the case whenever $n$ is odd.
Restricting to even $n$, the chain $\{H_{n} : n = 0, 2, 4, ...\}$ is aperiodic and takes values in $\{0, 2, 4, \dots\}$.
It is also positive recurrent, as words sampled from the PCFG have finite expected length \citep[2.2.1]{skachkova2018closing}.
Therefore, the Markov chain $\{H_{n} : n = 0, 2, 4, ...\}$ converges to its stationary distribution~\cite{mitzenmacherprobability}, which -- by positive recurrence -- must assign some nonzero weight $\pi_{2i}$ to each height  $2i$ ($i \geq 0$).
Hence, even when $n$ is even, the prefix $x$ is unbalanced with nonzero probability $1-\pi_0$ asymptotically independent of $n$.
Also, since the transition probabilities $P(H_{n+1}|H_n)$ are independent of $n$, there is an asymptotically constant nonzero probability that $x_1\dots x_{|x|-1}$ has height larger than $x$, i.e., the last bracket of $x$ is a closing one.}
After such an $x$, the next symbol is a closing bracket with constant nonzero probability $(1-p)$.
If $x$ can be followed by, say, `)' but not `]', then there is a string $x'$, differing only in one input position, but for which the next symbol can be `]' but not `)'.
As $x$ was assumed to end with a closing bracket, the exchanged symbol is not the last symbol of $x$, and thus the transformer's predictions on $x$ and $x'$ differ only by $\mathcal{O}(\frac{1}{n})$.
We can decompose the prediction task into predicting (1) whether an opening or closing bracket, or \textsc{EndOfSequence}, follows, and (2) whether a round or square bracket follows, in case a bracket follows.
The cross-entropy loss is the sum of the cross-entropies incurred on these two successive predictions tasks.
Therefore, when such a prefix $x$ is followed by the correct closing bracket, say `)', the model will incur, as $n \rightarrow \infty$, a cross-entropy loss on the second task of at least $\log 2$, reflecting at-chance performance in choosing between the two possible closing brackets.
In contrast, optimal cross-entropy loss on (2) would be $0$, as the bracket type (round or square) is actually fully determined by $x$.
Thus, the overall cross-entropy on all prefixes $x$ of length $n$ is, asymptotically as $n \rightarrow \infty$, at least $ P_0 \cdot (1-p) \cdot \log 2 > 0$ more than the optimal cross-entropy.
\end{proof}

We proceed to proving Lemma~\ref{lem:soft-tech}.
\begin{proof}[Proof of Lemma~\ref{lem:soft-tech}]
We compare the activations at the decoder layer for two inputs that only differ in the input at the $i$-th position.
Let $D = \|v_i-v_i'\|_2$ the norm of the difference of the input embeddings at this position.
	We show by induction over $k = 1, \dots, L$ that, for some $C > 0$ (chosen below) the differences between the resulting activations $y_j^{(k)}$, ${y_j^{(k)}}'$ are bounded as:
\begin{align*}
	\|y_i^{(k)}-{y_i^{(k)}}'\| &\leq C^{2k}D = \mathcal{O}(1) \\
	\|y_j^{(k)}-{y_j^{(k)}}'\| &\leq \frac{H^k C^{2k}D}{n} = \mathcal{O}(1/n)\ \ \ (j \neq i)
	\end{align*}
Once we have shown this, we know that the influence of any individual input on the final prediction is $\mathcal{O}(\frac{1}{n})$, with constants depending on the norms of parameter matrices and the number of layers.

At this point, it is worth remarking that a key property of transformers for this proof is that the number $L$ of layers is bounded independently of the input length.
A similar proof strategy can also be applied to other fixed-depth architectures that combine unboundedly many inputs in a smooth manner, such as 1D temporal convolutions with average pooling.

For $k=0$, $\|y_i^{(0)} - {y_i^{(0)}}'\| \leq D$, %\footnote{We are assuming that $f$ is addition or concatenation \cite{vaswani2017attention}; for operations such as an MLP, there would be an additional Lipschitz constant depending on parameters and activation functions.}
and %, where $L_f$ is the Lipschitz constant of the operation $f$ combining position and input embeddings.
$\|y_j^{(0)} - {y_j^{(0)}}'\| = 0$ for $j \neq i$.
For $k>0$, we first note that $\|y_j^{(k)}\|_2 \leq 2 C_{f^{act}}^{L}  (\|p_j\| + \|v_j\|)$,
where $C_{f^{act}} < \infty$ depends on the norms of the parameter matrices of $f^{act}$, which is implemented as a ReLU MLP \cite{vaswani2017attention}.
We'll write $F$ for this upper bound for $\|y_j^{(k)}\|_2$.
Attention logits are bounded by $A := F^2 C_{f^{att}}$ in the case of dot-product attention, and $A := 2 F C_{f^{att}}$ in the case of additive attention.
%If attention logits are bounded as $|a_i| \leq A$,
Then any attention weight $\widehat{a}_{j,i} = \exp(a_i)/\sum_j \exp(a_j)$ is upper bounded by $\frac{\exp(A)}{\exp(A) + (n-1) \exp(-A)} \leq \frac{\exp(2A)}{n-1}$.

Choose $C : = 2\cdot (1 + \exp(2A) + L_{f^{act}})$, where $L_{f^{act}}$ is the Lipschitz constant of the ReLU MLP $f^{act}$. 
Recall that activations $y_i^{(k)}$ are defined as $f^{act}(y_i^{(k-1)}, b_{i,k,1}, \dots, b_{i,k,H})$, where $b_{i,k,h}$ equals $\sum_{j=1}^n \hat{a}_{i,j}^{k,h} y_j^{(k-1)}$.
We first calculate
\begin{align*}
& \|b_{j,k,h} - b_{j,k,h}'\|  \leq \sum_{w=1}^n \hat{a}_{j,w}^{k,h} \|y_w^{(k-1)} - {{y}_w^{(k-1)}}'\|
\\
& = \hat{a}_{j,i}^{k,h} \|y_i^{(k-1)} - {{y}_i^{(k-1)}}'\|  + \sum_{w \neq i} {\hat{a}_{j,w}}^{k,h} \|y_w^{(k-1)} - {{y}_w^{(k-1)}}'\|
\end{align*}
which, using the induction hypothesis, is at most:
\begin{align*}
\frac{\exp(2A)}{n-1}  C^{2(k-1)} D + \frac{H^{k-1}C^{2(k-1)}D}{n} \leq \frac{H^{k-1} C^{2k-1} D}{n}
\end{align*}
Plugging this into the definition of $y_i^{(k)}$, the difference $\|y_j^{(k)} - {y_j^{(k)}}\|$ is at most
\begin{align*}
	L_{f^{act}} \cdot \left(\|y_j^{(k-1)}-{y_j^{(k-1)}}'\| + \sum_{q=1}^H \|b_{i,k,q} - b_{i,k,q}'\|\right)
\end{align*}
First, if $j= i$, this is bounded by (as $n \rightarrow \infty$)
\begin{align*}
\leq L_{f^{act}} \cdot \left(C^{2(k-1)}D + o(1)\right) \leq C^{2k}D
\end{align*}
Second, if $j\neq i$, it is bounded above by
\begin{align*}
	&  (1/n) \cdot L_{f^{act}} \cdot \left(H^{(k-1)} C^{2(k-1)}D + H^{k} C^{2k-1} D\right)
  %&   \\
\end{align*} 
which is bounded by $\leq  \frac{H^{k} C^{2k} D}{n}$.
This proves the inductive step for $k>0$.
\end{proof}

\section{Discussion}\label{sec:discussion}

We have shown that, even with infinite precision, transformers cannot robustly model non-counter-free regular languages, nor basic hierarchical structure.
In the hard attention setting, our results hold independently of activation functions and the magnitude of the parameters, and show that no transformer   can accurately classify strings as belonging to such languages.
In the soft attention setting, our results are slightly less general, but still show that transformers cannot achieve perfect cross-entropies when modeling distributions over these formal languages.

Our results are asymptotic, in the sense that they show that any transformer will make mistakes on modeling \textsc{Parity} and \textsc{2Dyck} when the input is \emph{sufficiently long}.
A transformer may nonetheless be able to perform well on on short inputs; indeed, given any bound $N$ on the input length, it is possible to construct a transformer that will achieve perfect accuracy or cross-entropy on all examples of length $n \leq N$; our results show that the number of heads and layers, or the parameter norms, will have to increase with $N$.
Practical implementations of transformers might thus be able circumvent such asymptotic limitations by using large numbers of layers and heads, in relation to the sentence lengths typically occurring in natural language.
Therefore, pending tighter nonasymptotic bounds, the results reported here need not constitute conclusive evidence for practical limitations of real-world NLP systems.

We believe that the most imminent implications of our results are theoretical in nature.
They showcase mathematical techniques for analyzing the capabilities of self-attention, an architecture at the heart of recent advances in NLP.
These tools provide theoretical understanding of differences between self-attention and theoretically more well-studied recurrent architectures:
Recurrent networks such as LSTMs can perfectly emulate finite-state automata, and therefore can model any finite state language with optimal cross-entropy, as long as the state transition and symbol emission distributions are Markovian.
In particular, \textsc{Parity} of i.i.d. bitstrings can be predicted with perfect accuracy and cross-entropy, independent of the input length.
Furthermore, infinite-precision RNNs and LSTMs can model stacks \cite{tabor2000fractal,gruning2006stack,kirov2012processing} and thus are theoretically capable of modeling \textsc{2Dyck} and other deterministic context-free languages perfectly.
The results presented here thus theoretically confirm the intuition that models entirely built on self-attention may have restricted expressivity when compared to recurrent architectures~\cite{tran2018importance,dehghani2018universal,shen2018disan,chen2018best,hao2019modeling}.
Complementing the asymptotic methods developed here with empirical studies or non-asymptotic extensions is an interesting avenue for future research.

While finite languages are sufficient to model language up to any finite bound on sequence length, it has typically been argued that asymptotically more powerful formalisms at the level of context-free grammars or beyond are necessary to properly capture generalizations about the syntax and meaning of natural language~(e.g., \citet{chomsky1957syntactic,shieber1985evidence}).
Our results entail that self-attention is limited in its ability to model context-free languages or evaluate logical formulas.
In particular, self-attention cannot in general emulate stacks or arbitrary finite-state automata.
Whether this hinders its capacity for syntactic generalization in practice is an interesting question; empirical research suggests that models with strong quantitative performance -- both recurrent and transformer models -- continue to struggle with syntactic generalization and that quantitative performance metrics such as perplexity can partly be dissociated from syntactic knowledge displayed on more challenging benchmarks (e.g., \citet{kuncoro2018lstms,marvin2018targeted, tran2018importance,mccoy2019berts}).

Nonetheless, the success of transformers across NLP tasks suggests that many aspects of natural language can be modeled well with methods that are formally too weak for the formal languages typically assumed in theoretical linguistics.
Beyond general limitations of asymptotic analysis, a possible reason for this phenomenon is that language uses recursive structure only in restricted ways due to cognitive factors.
For instance, it has long been noted that center embeddings, syntactic structures exhibiting iterated bracketing, are very challenging for humans to process \cite{miller-finitary-1963,gibson1999memory}.
Intriguingly, self-attention bears some resemblance to psycholinguistic models of memory in human sentence processing that assume that humans, while processing a word, attend to chunks that were stored in memory when processing some previous words \cite{lewis2005activation,parker2017cue}.
Such processing models predict difficulty with center embedding because they cannot count brackets \cite{lewis2005activation}, akin to what we have shown theoretically for neural network models based on self-attention.

\section{Conclusion}
We formally investigated the capabilities of self-attention in modeling regular languages and hierarchical structure.
We showed that transformers cannot model periodic regular languages or basic recursion, either with hard or soft attention, and even if infinite precision is allowed. %, or Boolean formula evaluation
This entails that self-attention cannot in general emulate stacks or general finite-state automata.
Our results theoretically confirm the idea that self-attention, by avoiding recurrence, has quite limited computational power.

\section*{Acknowledgments}
I thank Dan Jurafsky, Yoav Goldberg, the anonymous reviewers, and the members of the Stanford NLP Group for helpful comments.

\bibliography{literature}
\bibliographystyle{acl_natbib}

\end{document}